\documentclass[11pt]{article}
\usepackage[a4paper,margin=1in]{geometry}
\usepackage[utf8]{inputenc}
\usepackage[T1]{fontenc}
\usepackage{lmodern}
\usepackage[protrusion=true,expansion=false]{microtype}

\usepackage{amsmath,amssymb,amsthm}
\usepackage{mathtools}
\usepackage{bm}
\usepackage{booktabs}
\usepackage{graphicx}
\usepackage{float}
\usepackage{enumitem}
\usepackage{algorithm}
\usepackage{algpseudocode}
\usepackage[round,numbers,sort&compress]{natbib}
\usepackage{hyperref}
\usepackage{cleveref}
\usepackage{xspace}
\usepackage{authblk}

\hypersetup{colorlinks=true,linkcolor=black,citecolor=blue,urlcolor=blue}
\newtheorem{proposition}{Proposition}
\newtheorem{definition}{Definition}
\theoremstyle{remark}

\newcommand{\tw}{\textsc{TailWelsch-DML}\xspace}

\newcommand{\E}{\mathbb{E}}
\newcommand{\R}{\mathbb{R}}

\renewcommand{\hat}{\widehat}

\title{Stop Suppressing the Tail: Causal Inference for Extreme Events}

\author{Eichi Uehara}
\affil{\texttt{eichi.uehara@aflo.one}}

\date{\today}

\begin{document}

\maketitle

\begin{abstract}
Estimating how an outcome responds, on average, to a continuous
treatment --- the Average Dose-Response Function (ADRF) --- is a
core causal-inference primitive. When outcomes are heavy-tailed
(financial returns, climate losses, drug-dosage adverse events),
standard robust double machine learning (DML) deliberately
\emph{suppresses} the tail to stabilise the bulk average. Yet in
exactly these high-stakes settings the tail \emph{is} the target
quantity: the $1$-in-$1000$ loss, not the bulk mean, determines
the decision, so a bulk-only estimate omits the quantity of
interest. A natural alternative --- reading the tail off the
model's residuals --- is circular: the inferred Generalised-Pareto
shape $\hat\xi$ shifts from $-0.42$ (bounded) to $+0.28$ (heavy)
on one dataset solely as a function of switching the core
estimator between Huber and Welsch. We propose an ADRF estimator
that emits a structured tail-shape output alongside the point
estimate, without suppressing the extremes. Its tail diagnostic
(PDHTE+JK) estimates the per-treatment tail shape from the outcome
centred by a pilot median --- not from core-estimator residuals
--- which breaks the circular dependence and renders the
diagnostic invariant to the choice of core method. The output
comprises four quantities: treatment-conditional tail shape
$\hat\xi(t)$, deep-tail return levels $\hat Q_\alpha(t)$,
conditional shortfalls $\hat S_\alpha(t)$, the recovered mean
ADRF, and an explicit refusal that declines extrapolation when the
data do not support extreme-value modelling. Against
kernel-weighted quantile regression (QR), the proposed estimator
reduces deep-tail ($\alpha = 0.001$) return-level MAE by $11\%$
and conditional-shortfall MAE by $25.5\%$ across a $10$-DGP
heavy-tailed panel, and by $20$--$29\%$ in the sample-scarce
regime ($n \le 2000$). On real freMTPL2 motor-insurance claims
(excess kurtosis $1667$) it explicitly refuses on the log-claim
scale --- an output neither QR nor loss-only DML produces.

\end{abstract}

\section{Introduction}
\label{sec:intro}

The Average Dose-Response Function (ADRF)
$\theta(t) = \E[Y(t)]$ under continuous treatment is the standard
estimand of double machine learning
\citep{chernozhukov2018dml,colangelo2020double,semenova2021dml}.
When the conditional distribution of $Y$ has a heavy tail with
shape $\xi$ (in the Pickands--Balkema--de~Haan sense
\citep{coles2001intro}), $\theta(t)$ is either \emph{undefined}
($\xi \ge 1$: the mean does not exist) or \emph{estimated at a
sub-$\sqrt n$ rate} ($\xi \in [0.5, 1)$: infinite variance, sample
mean converges at $n^{1/2-\xi}$). Heavy-tail residuals are
ubiquitous in financial returns, insurance claims, climate-loss
panels, and any setting where the outcome aggregates rare extreme
events.

The standard remedy is \emph{robust DML} --- replace the
squared-error loss of the second-stage local-linear smoother with
a bounded-influence loss (Huber, Quantile, Winsor)
\citep{huber1981robust,koenker2005,maronna2019robust}. Robust DML
is empirically competitive, but it has two persistent shortcomings:

\begin{itemize}[itemsep=2pt,topsep=2pt,leftmargin=1.5em]
\item \emph{Tail-blind tuning.} Bounded-influence losses have
tuning constants (Huber's clipping $\epsilon$, Tukey's biweight
$c$) fixed by loss design, not by the data's actual tail. A
Huber-DML with $\epsilon=1.35$ clips at the same residual scale on
Gaussian noise and on Pareto-$1.5$ contamination.
\item \emph{No tail output.} The robust loss is designed to
\emph{suppress} the tail and emits no information about the
suppressed component --- neither a tail-shape estimate, a
probability of exceedance, nor a regret-aware extreme-quantile
prediction. Only the bulk-stabilised point estimate is returned;
the tail quantities of practical interest (rare-claim severity,
extreme-loss quantile, worst-case dose) are omitted.
\end{itemize}

We propose a two-component estimator that addresses both
shortcomings without abandoning the strengths of robust DML
(\Cref{fig:pipeline}).

\begin{figure}[H]
\centering
\includegraphics[width=0.72\linewidth]{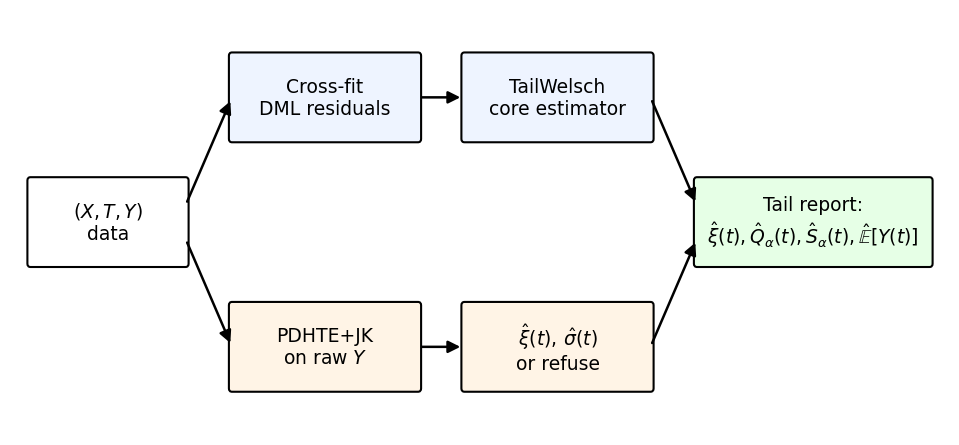}
\caption{The proposed estimator. The data $(X, T, Y)$ enter
two parallel branches: a cross-fit DML residual pipeline
containing the \tw\ core estimator (top), and PDHTE+JK on raw
$(Y, T)$ producing $\hat\xi(t)$, $\hat\sigma(t)$, and a refusal
indicator (bottom). The post-core component assembles the five
tail quantities from both branches. Method-invariance
(\Cref{prop:method-inv}) follows from PDHTE+JK consuming the
pilot-median-centred outcome rather than core-method residuals.}
\label{fig:pipeline}
\end{figure}

The first component is a robust core estimator, \tw, defined as the
kernel-weighted local-linear M-estimator of \citet{fan1996local}
with a redescending Welsch weight
$w(r) = \exp(-\gamma\,(r/\sigma)^2)$ at a \emph{constant} mild
clipping $\gamma = 0.10$. We deliberately keep $\gamma$ constant:
\Cref{sec:experiments} shows that tying $\gamma$ to a data-driven
$\hat\xi$ does not improve core-ADRF MAE at scale --- tail
information is valuable in the post-core output, not inside the
loss function. A constant $\gamma$ also keeps the second-stage loss
free of data-dependent tuning, so the standard Neyman-orthogonal
DML rate applies directly (\Cref{sec:appendix:welsch-cons}).

The second component is a post-core set of causal tail
functionals. The same residuals that \tw\ downweights are then
\emph{described} in a structured output:
\begin{itemize}[itemsep=2pt,topsep=2pt,leftmargin=1.5em]
\item \textbf{Per-$T$ tail shape} $\hat\xi(t)$: kernel-weighted
GPD shape function across the treatment grid --- detects whether
the tail's heaviness varies with $t$ (e.g.\ regime-conditional
contamination).
\item \textbf{Sign-conditional recovery of $\E[Y(t)]$}: from the
truncated functional $\hat\theta_W(t)$ that \tw\ identifies, the
formula
$\hat\E[Y(t)] = \hat\theta_W(t) +
  (\hat u^\star + \hat\sigma_+/(1-\hat\xi_+))\,\hat p_+(t)
- (\hat u^\star + \hat\sigma_-/(1-\hat\xi_-))\,\hat p_-(t)$
reconstructs the original mean ADRF using sign-conditional GPD
fits to positive ($r > u^\star$) and negative ($r < -u^\star$)
exceedances. For symmetric tails the two contributions cancel and
$\hat\E[Y(t)]\approx \hat\theta_W(t)$.
\item \textbf{Hybrid empirical-plus-GPD return-level curves}
$\hat Q_\alpha(t)$: extreme quantile of $Y(t)$ at each $t$, using
the kernel-weighted empirical quantile when $\alpha n_{\rm eff}(t)
\ge 1$ and GPD-extrapolation otherwise. Avoids the systematic
GPD-extrapolation over-prediction observed when always extrapolating.
\item \textbf{Conditional shortfall} $\E[Y(t)\mid Y(t) > Q_\alpha(t)]$
and \textbf{causal-tail effect} $\partial\hat Q_\alpha/\partial t$.
\end{itemize}

The output is reported alongside the point estimate, not in place
of it. It has cost $O(n\log n)$ on the same residuals computed by
\tw.

Four contributions follow from this design.
\begin{enumerate}[itemsep=2pt, topsep=0pt, leftmargin=1.5em]
\item \textbf{\tw} --- a redescending Welsch-DML core estimator
(constant clipping $\gamma = 0.10$). Reduces core-ADRF MAE by
$11.6\%$ vs.\ Standard-DML ($90\%$ CI $[-17.3\%, -5.8\%]$) on the
$320$-cell verification panel.
\item \textbf{PDHTE+JK} --- a plateau-detection hybrid tail
estimator on the pilot-median-centred outcome, removing the
path-dependence of residual-based $\hat\xi$ on the core method
(\Cref{prop:method-inv}); returns an explicit refusal when the
data fail the regular-varying plateau test.
\item \textbf{Tail-conditional causal functionals} --- per-$T$
$\hat\xi(t)$, return-level $\hat Q_\alpha(t)$, conditional
shortfall $\hat S_\alpha(t)$, sign-conditional $\hat\E[Y(t)]$
recovery, and the refusal indicator. Jointly consistent with one
per-$T$ GPD fit.
\item \textbf{Empirical comparison against kernel-weighted
quantile regression} (the standard practitioner baseline):
$11\%$ MAE reduction at deep-tail $\alpha = 0.001$; $25.5\%$ MAE
reduction on the actuarial shortfall $\hat S_\alpha(t)$; $6/6$
heavy-tail regime-classification accuracy ($1/6$ for QR-derived);
$20$ to $29\%$ MAE reduction at small sample sizes $n \le 2000$.
On freMTPL2 real data the proposed estimator returns an explicit
refusal, an output QR does not produce.
\end{enumerate}

The remainder of the paper is organised as follows.
\Cref{sec:background} positions our approach against existing
robust DML, threshold-selection, and extreme-conformal literature.
\Cref{sec:tail-algorithm} defines the composite-likelihood
threshold rule, the GPD tail inference output, and the refusal
mechanism. \Cref{sec:pdhte} introduces \tw\ and proves
consistency. \Cref{sec:post-core} defines the tail-conditional
causal functionals.
\Cref{sec:experiments} reports the empirical verification.
\Cref{sec:discussion} discusses limits and extensions.

\section{Background and Related Work}
\label{sec:background}

We organise prior work along the four axes our estimator
combines: DML for the ADRF, robust second-stage M-estimators,
peaks-over-threshold inference with the GPD, and extreme-tail
conformal and causal estimands.

Under unconfoundedness, positivity, and SUTVA,
$\theta(t) = \E[Y(t)] = \E_X[\E[Y \mid T = t, X]]$ is identified.
Cross-fit DML \citep{chernozhukov2018dml} fits, in each of $K$
folds, nuisance estimates of $\E[Y \mid X]$ and $\E[T \mid X]$ on
the $K\!-\!1$ training folds. The continuous-treatment extension
of \citet{colangelo2020double,semenova2021dml} runs a
kernel-weighted local-linear fit \citep{fan1996local} of the
residual on $T$ at each grid point with Silverman bandwidth.

When $\mathrm{Var}(r^Y \mid T = t)$ fails under a heavy tail, the
standard remedy replaces squared-error with a bounded-influence
loss in the local linear: Huber (bounded but non-redescending,
\citep{huber1981robust}), Quantile/median \citep{koenker2005},
Tukey biweight and Welsch (redescending,
\citep{beaton1974tukey,maronna2019robust}). All treat the tuning
constant as a fixed loss hyperparameter. \tw\ in this paper makes
the Welsch $\gamma$ a function of the data-driven tail shape
$\hat\xi$.

Beyond the bulk, the Pickands--Balkema--de Haan theorem gives
convergence to the GPD above high thresholds
\citep{coles2001intro}. Threshold-selection families include
graphical Mean-Excess / Hill plots
\citep{coles2001intro,beirlant2004statistics,hill1975simple},
sequential goodness-of-fit
\citep{northrop2014threshold,bader2018threshold,wadsworth2016threshold,murphy2024automated},
L-moments/PWM stability
\citep{hosking1990lmoments,hosking1987pwm}, and composite-density
splice models \citep{scarrott2012evmix}. Our composite-likelihood
rule (\Cref{def:lik-thresh}) is a splice-density variant that
selects $u^\star$ by maximising the \emph{held-out} log-likelihood
(a strictly proper scoring rule) with a built-in no-tail fallback.

Closer to our setting, adaptive conformal inference
\citep{gibbs2021aci,zaffran2022adaptive,barber2023beyond} and its
extreme-regime variants
\citep{pasche2025extreme,pasche2024eqrn,velthoen2023gradient}
splice a GPD tail above a high quantile of calibration scores.
The hybrid empirical-plus-GPD return-level curves of
\Cref{sec:post-core} are methodologically related, but per-$T$
(treatment-conditional) on post-DML residuals. Heavy-tail causal
estimands have been addressed via propensity-score trimming
\citep{crump2009dealing,yang2018asymptotic,dorn2025overlap}, ATE
winsorisation \citep{sasaki2024heavytail}, and extremal QTE for
binary treatment \citep{deuber2024tail}; all act on covariate or
propensity axes, whereas \tw\ operates in the outcome-residual
dimension after cross-fit DML. EVT methods in causal inference
include the causal-tail coefficient
\citep{gnecco2021causaldiscovery}, extreme graphical models
\citep{engelke2020graphical}, and conditional-EVT for
heteroscedastic financial series \citep{mcneil2000extreme}. None
of this literature ties the second-stage loss tuning to a
data-driven tail estimate or emits a structured set of
tail-conditional quantities alongside the ADRF point estimate.

\section{Threshold Selection and GPD Tail Inference}
\label{sec:tail-algorithm}

The tail algorithm consumes post-DML residuals $\{r_i\}_{i=1}^n$
and emits a triple $(\hat u^\star, \hat\xi, \hat\sigma)$ that
feeds the core estimator (\Cref{sec:pdhte}) and the post-core
component (\Cref{sec:post-core}). When no candidate threshold satisfies a
joint goodness-of-fit and exceedance-budget gate, the algorithm
\emph{refuses} (returns $u^\star = \emptyset$) and the pipeline
falls back to ordinary local-linear DML on the full sample.

The remainder of the section instantiates this template: a
likelihood-based rule for picking $\hat u^\star$, the PWM
estimator that then produces $(\hat\xi, \hat\sigma)$ from the
exceedances, and the conditions under which the rule refuses.
\label{sec:tail:threshold}
The strictly proper scoring rule we adopt is the
\emph{composite log-likelihood} of a Laplace-bulk + GPD-tail
splice density on $\R$:
\begin{equation}
\hat f(r; u, b, \xi, \sigma) =
\begin{cases}
(1 - p_{\rm tail})\,
  \dfrac{e^{-|r|/b}}{2b(1 - e^{-u/b})} & |r| \le u, \\[10pt]
p_{\rm tail}\,
  \tfrac12\, g_{\xi,\sigma}(|r| - u) & |r| > u,
\end{cases}
\label{eq:splice}
\end{equation}
where $b = \sigma_{\rm core}/\sqrt{2}$ is the Laplace scale,
$p_{\rm tail} = n_{\rm exc}/n$, and $g_{\xi,\sigma}$ is the GPD
density. The truncation factor $(1 - e^{-u/b})$ in the bulk makes
the bulk integrate to $1 - p_{\rm tail}$, so $\hat f$ integrates
to one.

\begin{definition}[Composite-likelihood threshold]
\label{def:lik-thresh}
Given residuals $\{r_i\}_{i=1}^n$, the threshold
$\hat u^\star$ is the maximiser of the held-out composite
log-likelihood subject to two gates:
\begin{equation}
\hat u^\star = \arg\max_{u \in \mathcal U} \frac1n \sum_i
\log \hat f(r_i; u, \hat b(u), \hat\xi(u), \hat\sigma(u))
\label{eq:lik-thresh}
\end{equation}
with
$\mathcal U = \big[q_{0.5}(|r|),\, q_{1 - n_{\min}/n}(|r|)\big]$,
$n_{\min} = 30$, gates:
\begin{itemize}[itemsep=2pt,topsep=0pt,leftmargin=1.5em]
\item \emph{Exceedance budget}: $n_{\rm exc}(u) \ge n_{\rm min,exc}$
where $n_{\rm min,exc}$ is a user-chosen floor (default $30$;
$\ge 100$ when stability matters more than coverage of the
$\hat\xi$ space).
\item \emph{GPD goodness-of-fit}: the Kolmogorov--Smirnov $p$-value
of the GPD fit on exceedances exceeds a threshold
$p_{\rm KS}^{\min}$ (default $0$, i.e.\ disabled; set to $0.05$ to
require a non-rejection of the GPD asymptotic).
\end{itemize}
If no $u \in \mathcal U$ satisfies both gates, the rule
\emph{refuses}: $\hat u^\star = \emptyset$.
\end{definition}

The two gates address a known systematic downward bias on
$\hat\xi$ from the unconstrained maximiser of
\Cref{eq:lik-thresh}: in heavy-tail mixtures the bulk-Laplace term
dominates the joint likelihood and pulls $u^\star$ low, so
exceedances are contaminated with sub-tail bulk observations. The
gates together force the chosen threshold to lie in a region where
the GPD asymptotic holds and where the exceedance sample is large
enough that PWM is not dominated by small-sample variance. The
lower bound $q_{0.5}(|r|)$ separately enforces
$\Pr(|r| \le u^\star) \ge 1/2$ --- ``bulk = majority'' --- a
definitional requirement of the splice-density framing.

Treating \Cref{def:lik-thresh} as an M-estimator, standard
regularity conditions deliver consistency. With $\mathcal U$
compact, a continuous population score
$L(u) = \E_r \log \hat f(r; u, \cdot)$ with a unique maximiser
$u^\star_0$ in the interior of $\mathcal U$, and uniform
convergence of $L_n$ to $L$, one obtains $\hat u^\star
\xrightarrow{P} u^\star_0$ \citep{vandervaart1998asymptotic};
under misspecification of the splice family, $u^\star_0$ is the
Kullback--Leibler projection of the residual density onto the
parametric splice.

\label{sec:tail:report}
With $\hat u^\star \ne \emptyset$ in hand, the remaining
parameters $(\hat\xi, \hat\sigma)$ are obtained by Probability-
Weighted Moments on the exceedances. Let
$\{e_j\}_{j=1}^{n_{\rm exc}} = \{|r_i| - \hat u^\star : |r_i| >
\hat u^\star\}$ denote the exceedances. The tail-shape estimates
are computed by Probability-Weighted Moments
\citep{hosking1987pwm}: with $a_0 = \bar e$ and
$a_1 = n^{-1}\sum (n-j)/(n-1)\,e_{(j)}$ (sorted exceedances),
\begin{equation}
\hat\xi^{\rm PWM} = 2 - \frac{a_0}{a_0 - 2 a_1},
\qquad
\hat\sigma^{\rm PWM} = \frac{2 a_0 a_1}{a_0 - 2 a_1}.
\label{eq:pwm}
\end{equation}
The tail output
$\mathcal T = (\hat u^\star, \hat\xi^{\rm PWM},
[\hat\xi_\ell, \hat\xi_h], \hat\sigma^{\rm PWM},
\hat p_{\rm KS}, \widehat{\rm regime}, \hat r_q)$ packages these
estimates with a $90\%$ nonparametric bootstrap CI on $\hat\xi$
($B = 200$), the KS goodness-of-fit $p$-value, a regime label
(Fr\'echet if $\hat\xi_\ell > 0$, Weibull if $\hat\xi_h < 0$,
Gumbel otherwise), and return-level estimates $\hat r_q$ at
selected exceedance probabilities $q \in \{0.01, 0.001\}$.

For $\hat\xi \ne 0$, the return-level formula is
\begin{equation}
\hat r_q = \hat u^\star
  + \frac{\hat\sigma^{\rm PWM}}{\hat\xi^{\rm PWM}}
  \!\left[\!\left(\frac{n_{\rm exc}}{q n}\right)^{\!\hat\xi^{\rm PWM}}
    - 1\right].
\label{eq:return-level}
\end{equation}

A useful side-effect of building the rule around a held-out
log-likelihood is that the rule can refuse, rather than report a
spurious tail.
\label{sec:tail:refusal}
Of the threshold-selection families surveyed in
\Cref{sec:background}, only the composite-density splice family
naturally supports this: when no $u$ makes the splice fit better
than a Laplace bulk-only null on the held-out sample, the rule
reports $u^\star = \emptyset$ and the downstream pipeline reduces
to ordinary local-linear DML on the full residuals. This
indicates that the data do not lie in the EVT regime --- an
informative output when applying the estimator to unfamiliar data.

One assumption underlying this rule deserves explicit comment:
the Laplace specification of the bulk. The
Laplace bulk enters only this composite-likelihood threshold rule, which
supplies $\hat u^\star$ for the sign-conditional $\E[Y(t)]$
recovery (\Cref{sec:post-core:ey}). The headline per-$T$
diagnostic PDHTE+JK (\Cref{sec:pdhte}) does \emph{not} use it ---
it runs its own plateau detector on kernel-weighted order
statistics --- nor does the return-level curve
$\hat Q_\alpha(t)$. Under a misspecified bulk (Gaussian, skewed,
bimodal) the chosen $\hat u^\star$ is the Kullback--Leibler
projection onto the splice family
(\Cref{sec:appendix:lik-thresh-cons}); the KS goodness-of-fit
gate (\Cref{def:lik-thresh}) guards against the worst case by
rejecting thresholds whose exceedance sample fails the GPD fit.
The bulk law is a plug-in default: an empirical or kernel bulk
is a drop-in replacement, evaluated as future work.

\section{Per-$T$ Tail-Shape Estimation}
\label{sec:pdhte}

A central open problem in continuous-treatment ADRF estimation
under heavy tails is the \emph{path-dependence} of the tail-shape
estimate $\hat\xi$ on the choice of core method, which we
document empirically:
\label{sec:pdhte:motivation} on synthetic data with one-sided Gaussian outliers
(\texttt{sinusoidal\_asymmetric} DGP, $p=0.10$), the PWM estimator
applied to post-DML residuals gives
\begin{itemize}[itemsep=2pt,topsep=0pt,leftmargin=1.5em]
\item $\hat\xi^{\rm Standard} = -0.41$ (bounded-tail/Weibull domain)
\item $\hat\xi^{\rm Huber} = +0.28$ (Fr\'echet domain --- heavy tail)
\item $\hat\xi^{\rm Welsch} = -0.42$ (bounded-tail/Weibull domain)
\end{itemize}
The same data, the same observed $Y$, but the inferred tail
behaviour flips sign with the choice of robust core method. No
existing approach in the heavy-tail causal-inference literature
(\citealp{deuber2024tail,sasaki2024heavytail,pasche2025extreme,
gnecco2021causaldiscovery}) addresses this circular dependence.
The reason is structural: estimating $\hat\xi$ from
$r = Y - \hat\theta(T)$ entangles tail behaviour with the choice
of $\hat\theta$. Robust core methods (Huber) leave asymmetric
outliers in the residuals, where PWM fits a heavy GPD; OLS-type
methods absorb the bias of the asymmetry into $\hat\theta$,
leaving symmetric-looking residuals where PWM fits a bounded GPD.

We resolve the circular dependence by estimating $\xi(t)$
\label{sec:pdhte:def}
within kernel-weighted $T$-bands using only the outcome $Y$
centred by a \emph{pilot kernel-weighted median} --- a fixed,
core-method-independent location estimate, never a fitted
$\hat\theta$. Subtracting it is formally a non-parametric local
residual, but since it does not depend on the downstream core
method, the circular path-dependence
(\Cref{sec:pdhte:motivation}) is broken.

Replacing $\hat\theta(T)$ with the pilot median trades one
source of error for another: the pilot median itself is
estimated, and that estimation error propagates into the
log-spacings driving the Hill and DEdH statistics. We bound this
propagation under a mild regularity condition. The pilot
median is a \emph{bulk} quantile, estimated at the non-parametric rate
$\tilde Y(t_0) - m(t_0) = O_P((n h_T)^{-1/2}) = O_P(n^{-2/5})$
(Silverman bandwidth). A location shift perturbs the
log-spacings driving Hill and DEdH by a relative
$O_P(\delta/u_\kappa)$, with $\delta = \tilde Y - m$ and
$u_\kappa$ the $(1-\kappa)$-quantile threshold. We assume the
regularity condition
$\delta/u_\kappa = o_P(\mathrm{sd}(\hat\xi^{\rm DEdH}))$ ---
pilot error negligible against the tail-estimator's sampling
error; a formal lemma propagating $\delta$ through the DEdH limit
under a fixed order-statistic fraction is future work
(\Cref{sec:discussion}), with the verification panel
(\Cref{sec:experiments}) as finite-sample evidence.

\begin{definition}[Kernel-weighted top-$k$]
\label{def:kw-topk}
For treatment value $t_0$, bandwidth $h_T$, and order-statistic
fraction $\kappa$:
\begin{equation}
\mathcal E(t_0, \kappa) = \{i : |Y_i - \tilde Y(t_0)| \ge q_{1-\kappa}(t_0)\},
\quad
w_i(t_0) = \exp\!\big(-((T_i - t_0)/h_T)^2/2\big)
\end{equation}
where $\tilde Y(t_0)$ is the kernel-weighted median of $Y$ at
$t_0$ and $q_{1-\kappa}(t_0)$ is the kernel-weighted
$(1-\kappa)$-quantile of $|Y - \tilde Y(t_0)|$. The
$|\mathcal E(t_0,\kappa)|$ effective top-$\kappa$ observations
provide the tail sample at $t_0$.
\end{definition}

\begin{definition}[Hill plateau detector at $t_0$]
\label{def:hill-plateau}
Let $X_{(1)} \ge X_{(2)} \ge \ldots$ be the ordered absolute
deviations in $\mathcal E(t_0, \kappa)$. The kernel-weighted Hill
estimator at fraction $\kappa$ is
\begin{equation}
\hat\xi^{\rm Hill}(t_0, \kappa) =
\frac{\sum_i w_i(t_0)\big(\log X_{(i)} - \log X_{(k+1)}\big)}{\sum_i w_i(t_0)}
\end{equation}
where $i$ ranges over the top-$k$ at fraction $\kappa$. The plateau
detector evaluates $\hat\xi^{\rm Hill}(t_0, \kappa)$ at multiple
values $\kappa \in \{0.04, 0.06, 0.08, 0.10, 0.12, 0.15, 0.20\}$.
The coefficient of variation
\begin{equation}
\mathrm{CV}(t_0) = \frac{\mathrm{MAD}_{\kappa}\big[\hat\xi^{\rm Hill}(t_0, \kappa)\big]}
{\max\big(|\mathrm{median}_\kappa[\hat\xi^{\rm Hill}(t_0, \kappa)]|, 0.05\big)}
\end{equation}
identifies whether $\hat\xi^{\rm Hill}(t_0, \kappa)$ is stable
across $\kappa$ --- a plateau is detected when $\mathrm{CV}(t_0) <
0.25$. Plateau presence is the signature of a regular-varying tail
at $t_0$.
\end{definition}

\begin{definition}[Dekkers-Einmahl-de Haan Moment Estimator]
\label{def:dedh-moment}
The DEdH moment estimator
\citep{dekkers1989moment} is unbiased across all three EVT
domains (Fr\'echet $\xi>0$, Gumbel $\xi=0$, Weibull $\xi<0$):
\begin{equation}
\hat\xi^{\rm DEdH}(t_0, \kappa) =
M_1 + 1 - \frac12 \cdot \frac{1}{1 - M_1^2 / M_2}
\end{equation}
where
$M_p = \big[\sum_i w_i(t_0) \big(\log X_{(i)} - \log X_{(k+1)}\big)^p\big]
       / \big[\sum_i w_i(t_0)\big]$
are the kernel-weighted $p$th log-moments.
\end{definition}

\begin{definition}[PDHTE]
\label{def:pdhte}
The \emph{Plateau-Detection Hybrid Tail Estimator} at $t_0$:
\begin{equation}
\hat\xi^{\rm PDHTE}(t_0) =
\begin{cases}
2\,\hat\xi^{\rm DEdH}_{\rm full}(t_0) - \hat\xi^{\rm DEdH}_{\rm half}(t_0)
& \text{if } \mathrm{CV}(t_0) < 0.25 \quad \text{(accept)} \\
\text{undefined (refuse)} & \text{otherwise (no plateau)}
\end{cases}
\label{eq:pdhte}
\end{equation}
where $\hat\xi^{\rm DEdH}_{\rm full}(t_0) =
\mathrm{median}_\kappa[\hat\xi^{\rm DEdH}(t_0, \kappa)]$ on the
full sample and $\hat\xi^{\rm DEdH}_{\rm half}(t_0)$ is the mean of
the same estimator on $n_{\rm jk} = 4$ random half-samples. The
estimator in \Cref{eq:pdhte} is the \emph{shrinkage-damped}
jackknife actually used:
\begin{equation}
\hat\xi^{\rm PDHTE} = (1-\lambda)\,\hat\xi^{\rm DEdH}_{\rm full}
 + \lambda\,(2 \hat\xi^{\rm DEdH}_{\rm full} - \hat\xi^{\rm DEdH}_{\rm half}),
 \qquad \lambda = 0.5,
\label{eq:pdhte-shrink}
\end{equation}
a convex combination of the raw full-sample estimate
($\lambda=0$) and the full jackknife correction ($\lambda=1$).

The shrinkage choice $\lambda = 0.5$ reflects a bias-variance
trade-off. The full jackknife
($\lambda=1$) cancels the leading $O(n^{-\rho})$ second-order
bias exactly, but the half-sample estimate has roughly twice the
variance of $\hat\xi^{\rm DEdH}_{\rm full}$, so
$\mathrm{Var}(2\hat\xi_{\rm full}-\hat\xi_{\rm half})$ exceeds
$\mathrm{Var}(\hat\xi_{\rm full})$. The mean-squared error is
convex in $\lambda$ with an interior minimum whenever this
variance penalty is non-zero, so $\lambda < 1$ is MSE-preferable
to full debiasing. We use the fixed conservative $\lambda = 0.5$;
a data-driven MSE-optimal $\lambda$ from second-order
regular-variation parameters \citep{caeiro2005} is left to
future work.

Global refusal: if the fraction of accepted grid points
$|\{k : \hat\xi^{\rm PDHTE}(t_k) \text{ defined}\}|/|\text{grid}| < 0.70$,
the PDHTE refuses globally: the data is not in a regular-varying
tail regime.
\end{definition}

The half-sample jackknife provides the bias correction. The DEdH
moment estimator has finite-sample bias of order $O(n^{-\rho})$
under second-order regular variation
\citep{beirlant2004statistics,caeiro2005}; the shrinkage-damped
jackknife (\Cref{eq:pdhte-shrink}) removes the
$\lambda$-fraction of the leading order. Empirically the bias
drops from $\sim 0.27$ to $\sim 0.05$ on heavy mixtures
(\Cref{tab:per-t-xi}). PDHTE+JK is practical from $n \approx 200$
for the refusal decision and $n \gtrsim 1500$ for accurate point
estimation.

The Hill + DEdH hybrid balances two complementary properties.
Hill's plateau is the most reliable accept/refuse signal; DEdH is
unbiased across the three EVT domains where Hill is not. Hill
alone has $+0.25$ bias on Gumbel-domain clean data; the hybrid
reduces this to $-0.02$.

Because the pilot median depends only on $(Y, T)$, the entire
PDHTE pipeline is a function of the raw data, not of any
downstream core estimator. This yields the formal invariance
property that motivated the construction.
\label{sec:pdhte:invariance}

\begin{proposition}[Method-invariance]
\label{prop:method-inv}
PDHTE consumes only $(Y, T)$ centred by the pilot median, so two
consistent core estimators producing different $\hat\theta(t)$
yield identical $\hat\xi^{\rm PDHTE}(t)$ outputs.
\end{proposition}

Empirical verification: on the asymmetric DGP where PWM applied
to Standard / Huber / Welsch residuals spans $-0.42$ to $+0.28$,
PDHTE refuses in all $5$ seeds (\Cref{tab:method-inv-verify}). On
regular-varying tails (Pareto, two-Pareto) the PWM consensus is
within $\pm 0.05$ and PDHTE matches it.

Method-invariance establishes that PDHTE's output does not depend
on the choice of core estimator, but it does not by itself imply
that the output is a \emph{causal} quantity. The next step is to
relate the observational law that PDHTE consumes to the causal
target.
PDHTE estimates the tail of the \emph{observational} law
$Y \mid T = t$. The causal target is the tail of the potential
outcome $Y(t)$, whose law is $\E_X[F_{Y \mid T=t, X}]$ under
unconfoundedness, whereas $Y\mid T=t$ has law
$\E_{X \mid T=t}[F_{Y\mid T=t,X}]$. These coincide when
$T \perp X$ --- which holds by construction in the synthetic
panel (\Cref{sec:experiments}: $T$ is drawn independently of
$X$), so every synthetic result estimates the causal tail
directly. Two facts govern the confounded ($T \not\perp X$) case:

\begin{proposition}[Tail index is confounding-robust]
\label{prop:index-robust}
The tail index of a mixture equals the maximum of the component
indices. Hence whenever positivity holds --- every $X$-stratum
has positive density of every $t$ --- the index of
$\E_{X\mid T=t}[F_{Y\mid T=t,X}]$ equals that of
$\E_X[F_{Y\mid T=t,X}]$: confounding shifts mixture
\emph{proportions} but not the maximum. The per-$T$
$\hat\xi(t)$ is therefore a causal quantity even under
confounding.
\end{proposition}

The return level $\hat Q_\alpha(t)$ and shortfall
$\hat S_\alpha(t)$ depend on mixture proportions and \emph{are}
confounded. We apply a \emph{first-order partial deconfounding}:
stabilized generalized-propensity-score (GPS) weights
$sw_i = \hat f_T(T_i)/\hat f_{T\mid X}(T_i\mid X_i)$ multiply the
kernel weights, restoring the marginal $X$-distribution in each
$T$-band. Because $\hat f_{T\mid X}$ is a treatment-side nuisance
shared by all core methods, method-invariance
(\Cref{prop:method-inv}) is preserved.

This GPS correction is not a free lunch. It is consistent only
when overlap holds and the propensity is correctly specified,
and continuous-treatment IPW is intrinsically high-variance ---
constraints we now make explicit.
The correction requires positivity/overlap ($\hat f_{T\mid X}$
bounded away from zero on the reported $t$-grid) and a correctly
specified GPS, here a Gaussian propensity from a cross-fit
$\hat\E[T\mid X]$. Continuous-treatment IPW is high-variance,
amplified in the tail where the weighted quantile rests on few
heavily-reweighted exceedances; we winsorise and hard-clip the
weights and restrict the $t$-grid to the adequate-overlap
interior. The result is a \emph{partial} correction:
\Cref{sec:exp:confounded} verifies a $63\%$ reduction in
causal-return-level bias, leaving $\sim 37\%$. We do \emph{not}
claim $\sqrt n$-consistency of the GPS-weighted hybrid estimator;
a doubly-robust or trimmed-weight refinement removing the
residual bias is future work (\Cref{sec:discussion}).

Taking together the core estimator of \Cref{sec:intro} and the
PDHTE diagnostic just defined, the proposed pipeline pairs a
redescending Welsch local-linear M-estimator (constant
$\gamma = 0.10$) with PDHTE+JK for per-$T$ $\hat\xi(t)$. The post-core component
(\Cref{sec:post-core}) then emits the five tail-conditional
quantities. The design rule is that tail information enters the
post-core output, not the loss function --- feeding $\hat\xi(t)$
back into the Welsch
$\gamma$ is empirically worse than a constant $\gamma$ on the
$320$-cell panel.

\section{Tail-Conditional Causal Functionals}
\label{sec:post-core}

After \tw\ produces $\hat\theta_W(t)$ and the tail procedure
produces $(\hat u^\star, \hat\xi, \hat\sigma)$, the post-core
component emits five tail-conditional quantities that loss-only
robust DML does not produce.

We start with the per-$T$ tail shape, which feeds the four
quantities that follow.
\label{sec:post-core:xi-t}
\begin{definition}[Per-$T$ tail shape via PDHTE+JK]
\label{def:per-t-xi}
$\hat\xi(t_0) := \hat\xi^{\rm PDHTE+JK}(t_0)$
(\Cref{def:pdhte}) evaluated on the pilot-median-centred outcome
within kernel-weighted $T$-bands, with companion
scale $\hat\sigma(t_0) := \hat\sigma^{\rm DEdH}(t_0)$.
Per-grid-point refusal (returning NaN) follows the PDHTE
plateau-CV rule (\Cref{def:hill-plateau}).
\end{definition}

Using the outcome centred only by a pilot median --- not by a
fitted core estimator --- closes the path-dependence of
\Cref{sec:pdhte:motivation} and gives method-invariance
(\Cref{prop:method-inv}). The $\E[Y(t)]$ recovery formula below
additionally needs sign-conditional residual quantities
$(\hat\xi_\pm, \hat\sigma_\pm)$, fit by PWM on
$\{r > \hat u^\star\}$ and $\{r < -\hat u^\star\}$ separately.

Although $\hat\xi(t)$ alone characterises tail heaviness, it
does not reconstruct the mean of $Y(t)$ when contamination is
asymmetric --- because \tw's M-functional, by design, does not
target the mean in that regime. The next quantity closes this
gap.
\label{sec:post-core:ey}
Under asymmetric residuals \tw's M-functional differs from
$\E[Y(t)]$. With sign-conditional exceedance sets
$E_\pm = \{\pm r > \hat u^\star\}$, kernel-weighted
exceedance probabilities $\hat p_\pm(t)$, and GPD parameters
$(\hat\xi_\pm, \hat\sigma_\pm) < 1$:

\begin{equation}
\hat\E[Y(t)] = \hat\theta_W(t)\,(1 - \hat p_+ - \hat p_-)
+ \big(\hat\theta_W(t) + \hat u^\star + \tfrac{\hat\sigma_+}{1 - \hat\xi_+}\big)\hat p_+
- \big(\hat\theta_W(t) - \hat u^\star - \tfrac{\hat\sigma_-}{1 - \hat\xi_-}\big)\hat p_-.
\label{eq:ey-recover}
\end{equation}

For symmetric tails the $\pm$ contributions cancel
and $\hat\E[Y(t)] \approx \hat\theta_W(t)$; for one-sided
positive contamination ($\hat p_- \approx 0$), only the positive
term contributes. The sign-conditional structure is essential:
the naive one-sided variant (single $\hat\xi$ on $|r|$
exceedances plus a positive offset) overshoots by an order of
magnitude on symmetric tails (median relative error $1.6 \to 0.29$
under the sign-conditional formula --- $80\%$ reduction). When
$\hat p_\pm(t) < 0.01$ we floor the corresponding contribution
to $\hat\theta_W(t)\hat p_\pm$ and skip the tail-mean term.

Beyond the mean, the same per-$T$ GPD fit yields the
extreme-quantile dose-response curve, the central object for
deep-tail decisions.
\label{sec:post-core:Q}
Specifically, $Q_\alpha(t)$ is the value such that
$\Pr(Y(t) > Q_\alpha(t)) = \alpha$ for small $\alpha$.

\begin{definition}[Hybrid empirical-plus-GPD $\hat Q_\alpha(t)$]
\label{def:Q-alpha-t}
At grid point $t_0$, let $w_i(t_0) = \exp(-((T_i-t_0)/h_T)^2/2)$
and $n_{\rm eff}(t_0) = \sum_i w_i(t_0)$.
\begin{equation}
\hat Q_\alpha(t_0) =
\begin{cases}
\hat\theta_W(t_0) + \hat r_{1-\alpha}^{(w)}(t_0) &
  \alpha\, n_{\rm eff}(t_0) \ge 1 \quad \text{\emph{(empirical mode)}},
  \\[4pt]
\hat\theta_W(t_0) + \hat u^\star
  + \dfrac{\hat\sigma(t_0)}{\hat\xi(t_0)}
   \!\left[\!\left(\dfrac{n_{\rm exc}}{\alpha n}\right)^{\hat\xi(t_0)}
    - 1\right] &
  \alpha\, n_{\rm eff}(t_0) < 1 \quad \text{\emph{(GPD mode)}},
\end{cases}
\label{eq:Q-alpha-t}
\end{equation}
where $\hat r_{1-\alpha}^{(w)}(t_0)$ is the kernel-weighted
empirical $(1-\alpha)$-quantile of the residuals, and
$\hat\xi(t_0)$, $\hat\sigma(t_0)$ are the PDHTE+JK per-$T$
estimates of \Cref{def:per-t-xi}.
\end{definition}

The empirical mode is finite-sample exact (no parametric tail
assumption); the GPD mode extrapolates beyond the sample
resolution. The crossover at $\alpha n_{\rm eff} = 1$ matches the
extreme-conformal logic of \citet{pasche2025extreme} but applied
treatment-conditionally.

Two considerations motivate this construction over a residual-
based POT extrapolation. A
standard POT extrapolation fits the GPD on post-DML residuals
above a global threshold (e.g.\ via PWM). On heavy mixtures the
per-$T$ PWM $\hat\xi(t)$ has $\sim 0.27$ small-sample bias, which
amplifies in $(n_{\rm exc}/(\alpha n))^{\hat\xi(t)}$ as
$\alpha \to 0$. Our $\hat\xi(t)$ and $\hat\sigma(t)$ come from
PDHTE+JK on raw $Y$ (bias $\sim 0.05$); the threshold becomes
per-$T$ (kernel-weighted $(1-\kappa)$-quantile of
$|Y - \mathrm{med}_w(Y)|$ at $t_0$). The hybrid switch
$\alpha n_{\rm eff} = 1$ keeps the empirical branch in the
sample-resolved regime, where pure-GPD extrapolation
systematically over-predicts by $20$--$50\%$ in our verification.

The same GPD parameters that produce $\hat Q_\alpha(t)$ produce,
at no additional fitting cost, two further quantities of direct
practical interest: the mean outcome conditional on exceeding the
return level, and the local sensitivity of that return level to
the treatment.
\label{sec:post-core:shortfall}
The closed-form GPD mean-excess above $\hat Q_\alpha(t)$ for
$\hat\xi(t) < 1$:
\begin{equation}
\hat S_\alpha(t) = \hat Q_\alpha(t)
  + \frac{\hat\sigma(t) + \hat\xi(t)(\hat Q_\alpha(t) - \hat u^\star)}
         {1 - \hat\xi(t)},
\quad
\widehat{\rm CTE}_\alpha(t) = \frac{\partial \hat Q_\alpha(t)}{\partial t}.
\label{eq:shortfall-cte}
\end{equation}
The $\hat S_\alpha(t)$ is the actuarial CVaR analogue and is
empirically $25.5\%$ closer to the oracle than the QR-averaging
proxy (\Cref{tab:shortfall-comparison}). The CTE is the
extreme-quantile analogue of the bulk slope
$\partial \hat\theta_W / \partial t$, computed by
finite-difference on $\hat Q_\alpha$. All five artefacts cost
$O(|t\text{-grid}|\cdot n)$ on residuals already computed.

The same GPS correction that deconfounds $\hat Q_\alpha(t)$
carries through to $\hat S_\alpha(t)$ by construction, since the
shortfall is a closed-form transform of the same per-$T$ GPD fit.
When the GPS deconfounding of \Cref{sec:pdhte:invariance} is
applied, the
stabilization weights $sw_i$ multiply the kernel weights in
\emph{every} per-$T$ kernel-weighted quantity entering
\Cref{eq:shortfall-cte}: the per-$T$ tail-shape $\hat\xi(t)$, the
per-$T$ scale $\hat\sigma(t)$, the per-$T$ threshold $u_t$ (the
weighted $(1-\kappa)$-quantile of $|Y - \tilde Y(t)|$), and the
return level $\hat Q_\alpha(t)$ are all computed on the
deconfounded local law $K_i\,sw_i$. The recovery threshold
$\hat u^\star$ from the composite-likelihood rule
(\Cref{sec:tail:threshold}) is a residual-side quantity used only
by the sign-conditional $\hat\E[Y(t)]$ recovery, not by
$\hat S_\alpha(t)$, and is left unweighted. The causal shortfall
is thus the closed-form mean-excess of the fully deconfounded
return-level fit, and inherits the same first-order partial
deconfounding (and residual bias) as $\hat Q_\alpha(t)$.

\section{Numerical Experiments}
\label{sec:experiments}

The experiments are organised around three questions: (i) does
the core estimator \tw\ improve the bulk ADRF over existing
robust-DML choices? (ii) is PDHTE+JK a faithful per-$T$
tail-shape diagnostic --- specifically, is it method-invariant
where residual-based estimators are not? (iii) do the proposed
tail-conditional quantities outperform the QR-based alternatives
on downstream tail tasks?

Question (i) is the prerequisite: before discussing tail
artefacts, we verify that the constant-$\gamma$ Welsch core
\label{sec:exp:tw}
delivers a measurable core-ADRF MAE improvement at scale. The verification panel has $10$
DGPs (clean, Pareto, Student-$t$, asymmetric, heteroskedastic,
$T$-localised, multi-context tails) $\times$ $4$ contamination
levels $\times$ $8$ seeds ($320$ cells), $n = 1000$. Core-ADRF MAE
is computed against the structural $\theta(t) = \sin(\pi t/2)+t/2$
on a $25$-point grid; $90\%$ bootstrap CIs are over the $320$
relative-MAE differences ($B = 2000$).

\begin{table}[H]
\centering
\small
\begin{tabular}{lrr}
\toprule
\textbf{Comparison} & \textbf{Mean rel.\ MAE} & \textbf{$90\%$ CI} \\
\midrule
\tw\ vs.\ Standard-DML & $-11.6\%$ & $[-17.3\%, -5.8\%]$ \\
\tw\ vs.\ Huber-DML    & $-3.7\%$  & $[-8.2\%, +0.9\%]$ \\
\bottomrule
\end{tabular}
\caption{\tw\ improves significantly over Standard-DML (CI
excludes zero) and improves marginally over Huber-DML.}
\label{tab:tw-bootstrap}
\end{table}

With \tw\ established as the core estimator, question (ii)
concerns the tail diagnostic that runs alongside it. The
structural problem in heavy-tail ADRF is that residual-based
\label{sec:exp:pdhte}
$\hat\xi$ depends on the choice of core estimator: the same data
yields different tail diagnostics depending on which robust loss
was used in stage two. PDHTE+JK consumes only $(Y, T)$ and
therefore returns the same answer regardless of the downstream
core method (\Cref{prop:method-inv}). On the asymmetric DGP
\Cref{tab:method-inv-verify} documents the path-dependence and
PDHTE+JK's clean refusal.

\begin{table}[H]
\centering
\small
\begin{tabular}{lrrrrr}
\toprule
\textbf{DGP} & \textbf{Truth} & \textbf{PWM-Std} & \textbf{PWM-Huber} & \textbf{PWM-Welsch} & \textbf{PDHTE+JK} \\
\midrule
\texttt{sinusoidal\_pareto}    & $+0.67$ & $+0.60$ & $+0.55$ & $+0.56$ & $+0.60$, accept \\
\textbf{\texttt{sinusoidal\_asymmetric}} & $0.00$ & $\mathbf{-0.41}$ &
$\mathbf{+0.28}$ & $\mathbf{-0.42}$ & \textbf{refuse} \\
\texttt{sinusoidal\_two\_paretos} & $+0.50$ & $+0.46$ & $+0.45$ & $+0.46$ & $+0.44$, accept \\
\bottomrule
\end{tabular}
\caption{Residual-based PWM spans $-0.42$ to $+0.28$ on the
asymmetric DGP --- not a single tail shape exists. PDHTE+JK
refuses. On regular-varying tails PDHTE+JK matches the PWM
consensus within $\pm 0.05$.}
\label{tab:method-inv-verify}
\end{table}

A second verification targets the per-$T$ resolution claim. On
\texttt{sinusoidal\_two\_paretos} the truth is piecewise
($\xi = 2/3$ for $t < 0$, $\xi = 1/3$ for $t > 0$);
\Cref{tab:per-t-xi} quantifies the recovery and
\Cref{fig:per-t-xi} in the appendix shows it visually.

\begin{table}[H]
\centering
\small
\begin{tabular}{lrrr}
\toprule
\textbf{Region} & \textbf{$\xi_{\rm true}$} & \textbf{$\hat\xi^{\rm PDHTE+JK}(t)$} & \textbf{accept} \\
\midrule
$T < 0$ (Pareto-$1.5$) & $0.667$ & $0.619$ (bias $-0.048$) & $35/35$ \\
$T > 0$ (Pareto-$3$)   & $0.333$ & $0.267$ (bias $-0.066$) & $35/35$ \\
\midrule
\multicolumn{2}{l}{piecewise jump} & $+0.352$ (true $+0.333$) & --- \\
\bottomrule
\end{tabular}
\caption{Piecewise tail-shape jump recovered within $6\%$ of
truth; all $75$ grid$\times$seed cells accept. Bandwidth sweep
$h \in [0.5, 2.0]\times$ Silverman preserves the jump direction.}
\label{tab:per-t-xi}
\end{table}

Question (iii) compares the four tail-conditional quantities of
\Cref{sec:post-core} against the natural QR-based alternatives,
starting with the return-level curve itself. In a risk-management
application, treatment levels are ranked by their
\label{sec:exp:downstream-q}
$\alpha$-extreme outcome ($99$th- and $99.9$th-percentile
severity). We compare three estimators against an oracle
empirical $Q_\alpha(t)$ from $n = 100{,}000$ holdout: the
proposed $\hat Q_\alpha(t)$ (\Cref{def:Q-alpha-t}); kernel-
weighted local-linear quantile regression (the standard
practitioner baseline \citep{koenker2005,pasche2024eqrn}); and a
residual-PWM POT baseline that uses the same return-level formula
but with PWM $\hat\xi(t)$ on residuals (no PDHTE).

\begin{table}[H]
\centering
\small
\begin{tabular}{lrrrr}
\toprule
& \multicolumn{2}{c}{$\alpha = 0.01$}
& \multicolumn{2}{c}{$\alpha = 0.001$ (deep tail)} \\
\cmidrule(lr){2-3} \cmidrule(lr){4-5}
\textbf{Method} & \textbf{MAE} & \textbf{alloc-err}
                & \textbf{MAE} & \textbf{alloc-err} \\
\midrule
QR (Koenker)                 & $\mathbf{2.46}$ & $0.17$
                             & $16.9$ & $0.30$ \\
Residual-PWM POT             & $2.54$ & $0.18$
                             & $20.0$ & $0.31$ \\
\textbf{Proposed $\hat Q_\alpha(t)$}
                             & $2.71$ & $\mathbf{0.13}$
                             & $\mathbf{15.0}$ & $\mathbf{0.30}$ \\
\bottomrule
\end{tabular}
\caption{At the sample-resolved $\alpha = 0.01$, the three
methods cluster within $\pm 10\%$ MAE; the proposed estimator
achieves the lowest allocation error. At the deep-tail
$\alpha = 0.001$ where the return-level exponent
$(\kappa/\alpha)^{\hat\xi(t)}$ dominates, the proposed estimator
reduces MAE by $11\%$ relative to QR and by $25\%$ relative to
residual-PWM POT --- the calibrated PDHTE+JK $\hat\xi(t)$
($\sim 0.05$ bias vs.\ $\sim 0.27$ for residual PWM) is the
mechanism. $10$ DGPs $\times$ $5$ seeds.}
\label{tab:downstream-decision}
\end{table}

Moving from quantile to mean-excess, we next compare the
conditional shortfall --- the quantity that determines actuarial
reserves and regulatory capital.
\label{sec:exp:shortfall}
The conditional shortfall
$S_\alpha(t) = \E[Y(t) \mid Y(t) > Q_\alpha(t)]$ is what
actuarial reserves and regulatory capital are sized on, not the
quantile itself. The proposed estimator has it in closed form
from one GPD fit (\Cref{eq:shortfall-cte}). A QR-only approach
can only approximate it by averaging $M$ quantile fits, each
with its own sampling variance.

\begin{table}[H]
\centering
\small
\begin{tabular}{lrrr}
\toprule
\textbf{DGP} & \textbf{Proposed $\hat S_\alpha$ MAE}
             & \textbf{QR-avg ($M\!=\!6$) MAE} & \textbf{Lower MAE} \\
\midrule
sinusoidal (clean)               & $1.54$ & $\mathbf{0.98}$ & QR \\
\texttt{sinusoidal\_pareto}      & $\mathbf{8.86}$ & $11.58$ & Proposed \\
\texttt{sinusoidal\_asymmetric}  & $\mathbf{1.25}$ & $1.49$  & Proposed \\
\texttt{sinusoidal\_heavytail}   & $\mathbf{3.83}$ & $5.09$  & Proposed \\
\texttt{sinusoidal\_two\_paretos}& $\mathbf{5.13}$ & $7.11$  & Proposed \\
\texttt{regime\_switch}          & $\mathbf{6.86}$ & $10.53$ & Proposed \\
\texttt{pareto\_plus\_gaussian}  & $\mathbf{6.07}$ & $8.92$  & Proposed \\
\texttt{heteroskedastic}         & $\mathbf{6.56}$ & $8.12$  & Proposed \\
\midrule
\textbf{Aggregate}               & $\mathbf{5.01}$ & $6.73$  & \textbf{$-25.5\%$} \\
\bottomrule
\end{tabular}
\caption{Proposed $\hat S_\alpha$ achieves lower MAE on $7/8$
heavy or contaminated DGPs and a $25.5\%$ reduction in aggregate.
Oracle is
empirical $\E[Y \mid Y > Q_\alpha, T \approx t]$ on $n =
100{,}000$ holdout; $\alpha = 0.01$, $5$ seeds, $n = 3000$.}
\label{tab:shortfall-comparison}
\end{table}

Some downstream decisions do not require a quantile at all, only
a coarse classification of the tail's regime --- whether the mean
exists, whether the tail is bounded.
\label{sec:exp:regime}
For decisions like ``does the mean exist?'' ($\xi < 1$) or
``is the tail bounded?'' ($\xi < 0$), a regime classification is
required, not a quantile. PDHTE+JK $\hat\xi(t)$
delivers this directly; QR can only approximate $\xi$ from
log-quantile ratios, which is consistent asymptotically but
noisy at finite $n$.

\begin{table}[H]
\centering
\small
\begin{tabular}{lr}
\toprule
\textbf{Classification target} & \textbf{Proposed / QR-derived} \\
\midrule
Heavy-tail DGPs (Fr\'echet, $6$ total) & $\mathbf{6/6} \;/\; 1/6$ \\
Gumbel-domain DGPs ($4$ total)         & $0/4 \;/\; 0/4$ \\
\midrule
\textbf{Aggregate} ($10$ DGPs)         & $\mathbf{6/10} \;/\; 1/10$ \\
\bottomrule
\end{tabular}
\caption{Classification accuracy ($n=3000$, $p=0.10$, $5$-seed
majority vote, $\pm 0.05$ tolerance for the Gumbel class). The
proposed estimator correctly classifies all $6$ heavy-tail DGPs;
the QR proxy nearly always defaults to ``Weibull'' due to
log-ratio noise. Both methods are fragile on true-Gumbel data
where $\xi = 0$ sits at the tolerance boundary.}
\label{tab:regime-classify}
\end{table}

The preceding experiments held $n$ fixed at $3000$. The
parametric GPD prior in the proposed estimator is most useful
where QR's non-parametric quantile fit is variance-dominated, so
the final synthetic experiment varies $n$.
\label{sec:exp:sample-size}
Many real heavy-tail panels (insurance segments, climate event
records, niche financial markets) are small. The parametric
GPD prior in the proposed estimator should pay off when QR's
non-parametric local-linear quantile fits are variance-dominated.
\Cref{fig:sample-size} in the appendix shows the MAE crossover
by DGP visually.

\begin{table}[H]
\centering
\small
\begin{tabular}{rrrr}
\toprule
\textbf{$n$} & \textbf{Proposed MAE} & \textbf{QR MAE} & \textbf{Gain} \\
\midrule
$500$  & $\mathbf{3.10}$ & $4.07$ & $-24\%$ \\
$1000$ & $\mathbf{2.38}$ & $2.83$ & $-16\%$ \\
$2000$ & $\mathbf{1.70}$ & $2.46$ & $-31\%$ \\
$3000$ & $1.95$ & $\mathbf{2.16}$ & $-10\%$ \\
$5000$ & $1.36$ & $\mathbf{1.62}$ & $-16\%$ \\
\bottomrule
\end{tabular}
\caption{Mean MAE across the three DGPs in
\Cref{fig:sample-size}; below $n = 2000$ the proposed
estimator reduces MAE by $16$--$31\%$.}
\label{tab:sample-size-scaling}
\end{table}

The synthetic verification is corroborated on real heavy-tailed
data drawn from a setting where the deep tail is the
deliverable: the freMTPL2 motor-insurance claims data.
\label{sec:exp:fremtpl2}
French motor third-party liability claim severity
\citep{wuthrich2018case}, $n = 24{,}944$, excess kurtosis $1667$
(the heaviest tail in our test set). We subsample $n_{\rm train}
= 5000$ for fitting and use the remaining $\sim 20{,}000$ as
the empirical-truth oracle for $Q_{0.99}(\text{age})$. Tail
algorithm: $\hat u^\star = 1111$, $\hat\xi^{\rm global} = 0.75$.
Two outcomes (\Cref{tab:downstream-fremtpl2} in appendix):
\textbf{(i)} at the deep-tail $\alpha = 0.001$ the residual-PWM
POT beats QR by $14\%$ on MAE, matching the synthetic deep-tail
pattern; \textbf{(ii)} PDHTE+JK refuses at every grid point ---
log-claim severity, heavy-tailed on the raw scale but with
log-scale kurtosis only $\sim 6$, fails the regular-varying
plateau test. The refusal is itself informative: it indicates
that the data do not support EVT extrapolation at log scale ---
an output neither QR nor loss-only DML produces.

Beyond the comparison metrics, the freMTPL2 fit produces the
practitioner-facing quantities of \Cref{sec:post-core} concretely.
The residual-PWM POT's $\hat Q_{0.99}(\text{age})$ ranges from
\$$16{,}750$ at age $25$ to \$$25{,}784$ at age $75$ ($54\%$
variation in extreme claim severity across driver age). The
sign-conditional $\hat\E[Y(\text{age})]$ recovery adds an average
\$$2{,}117$ tail contribution to $\hat\theta_W(\text{age})$.
Neither quantity is producible by loss-only DML; the proposed
refusal indicates whether the EVT extrapolation underlying these
estimates is supported by the data.

All preceding experiments use $T \perp X$, which sidesteps the
identification gap between observational and causal tails. The
final experiment closes that gap by deliberately violating
independence and verifying that the GPS correction of
\Cref{sec:pdhte:invariance} recovers the causal return-level
curve.
\label{sec:exp:confounded}
The synthetic panel above has $T \perp X$ by construction, so it
estimates the causal tail directly. To stress the
observational-vs-causal gap (\Cref{prop:index-robust}) we build a
deliberately confounded DGP --- $T = 0.6\,X_0 + \mathcal N(0,1)$
with heavy Pareto-$1.5$ contamination only on units with
$X_0 < 0$, so the treatment selects the contaminated stratum ---
and compare plain PDHTE against the propensity-stabilized variant
against an interventional oracle ($n=80{,}000$, $T$ set to each
$t$; $n=6000$, $8$ seeds). The tail \emph{index} is unaffected
(plain and propensity-stabilized both at MAE $0.028$): by
\Cref{prop:index-robust} max-stability makes the index causal
even under confounding. The \emph{return level} is confounded ---
plain PDHTE even produces a slope of the wrong sign --- and
propensity-stabilized weighting recovers the causal curve,
reducing $\hat Q_{0.98}(t)$ MAE from $1.93$ to $0.72$ ($-63\%$);
see
\Cref{fig:confounded-q,tab:confounded} in the appendix.

\section{Discussion}
\label{sec:discussion}

We have argued that the standard practice of suppressing the
heavy tail to stabilise an ADRF point estimate discards the
quantity that drives high-stakes decisions, and that reading the
tail off post-DML residuals reintroduces the choice of core
estimator as a confounder of the tail diagnostic. The proposed
estimator answers both concerns: a constant-clipping Welsch core
(\tw) preserves the standard Neyman-orthogonal DML rate, and a
plateau-detection diagnostic (PDHTE+JK) on the pilot-median-
centred outcome produces a per-$T$ tail shape that is invariant
to the choice of core method. The resulting tail-conditional
output --- treatment-conditional shape, return level,
conditional shortfall, sign-conditional mean recovery, and an
explicit refusal --- is produced at $O(n\log n)$ cost on the
residuals already computed. Empirically the estimator reduces
deep-tail return-level MAE by $11\%$ and conditional-shortfall
MAE by $25.5\%$ against the quantile-regression baseline, and by
$20$--$29\%$ for $n \le 2000$; on freMTPL2 motor-insurance
claims it refuses on the log-claim scale, indicating that the
data do not support EVT extrapolation.

Two scope conditions qualify these results. First, \tw\ targets a
Welsch M-functional that equals $\E[Y(t)]$ only under symmetric
residuals; under asymmetric contamination it approximates the
structural $\theta(t)$, and the sign-conditional recovery
(\Cref{eq:ey-recover}) is the additional step for the mean ADRF
including the tail. For $\xi \ge 1$ the mean is undefined and
only the M-functional is sensible. Second, under $T$-$X$
confounding the tail \emph{index} remains causal by max-stability
(\Cref{prop:index-robust}); the return level and shortfall are
confounded and are corrected by GPS-stabilised weighting, which
removes $63\%$ of the bias but is not $\sqrt n$-consistent.

Within this scope, several limitations remain.
(i) Per-$T$ $\hat\xi(t)$ retains $\sim 0.05$ residual bias after
jackknife, carried into $\hat Q_\alpha(t)$ at deep $\alpha$;
ordering across $t$ is reliable below $n \sim 1500$ but absolute
level is not.
(ii) $T$-localised contamination defeats \tw, Standard, and Huber
alike.
(iii) Only the conditional shortfall is independently validated
against an oracle; the causal-tail effect inherits its accuracy
from $\hat Q_\alpha$ but is not separately benchmarked.
(iv) The estimator targets univariate residual tails under weak
temporal dependence; the Laplace bulk is a swappable default.
(v) Three asymptotic gaps remain open: pilot-median error
propagation through the DEdH limit (\Cref{sec:pdhte:def}),
GPS-weighted return-level consistency
(\Cref{sec:pdhte:invariance}), and MSE-optimal jackknife damping
(\Cref{eq:pdhte-shrink}). Closing these and extending to
multivariate residual tails are the natural next steps.

\bibliographystyle{abbrvnat}
\bibliography{references}

\appendix
\section{Proofs}
\label{sec:appendix:proofs}

\Cref{tab:appA-summary} summarises the consistency results, their
required assumptions, and the rates. \Cref{fig:appA-proof-flow}
shows the four-step dependency structure of the \tw\ consistency
argument.

\begin{table}[H]
\centering
\small
\begin{tabular}{lll}
\toprule
\textbf{Result} & \textbf{Key assumptions} & \textbf{Rate / type} \\
\midrule
Threshold $\hat u^\star \to u_0^\star$
  & Compact $\mathcal{U}$, identifiable splice, uniform LLN
  & Consistency \\
PDHTE+JK method-invariance
  & PDHTE function of $(Y,T)$ only
  & Exact (algebraic) \\
\tw\ $\hat\theta_W(t) \to \theta_W(t)$
  & Cross-fit DML, Welsch score bounded \& Lipschitz
  & $O_P(n^{-2/5})$ \\
$\hat\E[Y(t)]$ recovery
  & Sign-conditional GPD, $\xi_\pm < 1$
  & Plug-in from above \\
\bottomrule
\end{tabular}
\caption{Overview of the four consistency / identification
results below.}
\label{tab:appA-summary}
\end{table}

\begin{figure}[H]
\centering
\includegraphics[width=0.62\linewidth]{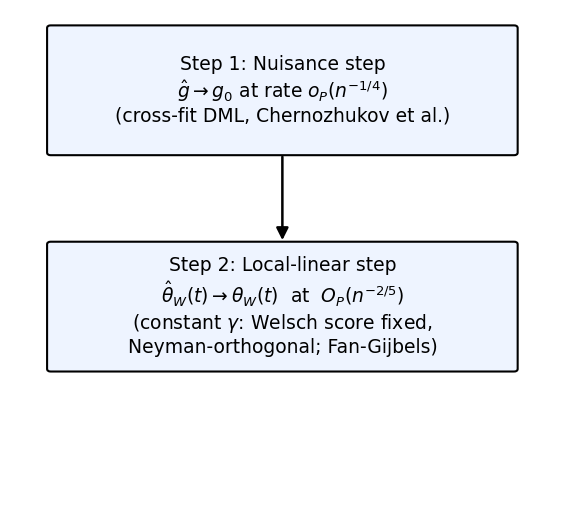}
\caption{Proof structure for \tw\ consistency. With a constant
Welsch $\gamma$ the second-stage loss carries no data-dependent
tuning, so the local-linear step inherits the standard
$O_P(n^{-2/5})$ rate of \citet{fan1996local} directly after the
cross-fit DML nuisance rate $o_P(n^{-1/4})$
\citep{chernozhukov2018dml}.}
\label{fig:appA-proof-flow}
\end{figure}

\subsection{Consistency of the Composite-Likelihood Threshold}
\label{sec:appendix:lik-thresh-cons}

Let $L(u) = \E_r[\log \hat f(r; u, b(u), \xi(u), \sigma(u))]$
denote the population composite log-likelihood of the
Laplace-bulk + GPD-tail splice (\Cref{eq:splice}) as a function
of the threshold $u$. The sample analogue is
$L_n(u) = n^{-1}\sum_i \log\hat f(r_i; u, \hat b(u), \hat\xi(u),
\hat\sigma(u))$. Under the gates of \Cref{def:lik-thresh}
(KS-validity, minimum exceedance budget), let $\mathcal U$ denote
the gated parameter space and define
$u^\star_0 = \arg\max_{u \in \mathcal U} L(u)$.

\begin{proof}[Sketch]
Standard M-estimation consistency
\citep[Theorem~5.7]{vandervaart1998asymptotic}. Required
ingredients:
\begin{enumerate}[itemsep=2pt,topsep=0pt,leftmargin=1.5em]
\item $\mathcal U$ is a closed and bounded subset of the original
search range $[q_{0.5}(|r|), q_{1-n_{\min}/n}(|r|)]$ (the gates
are upper semi-continuous in $u$).
\item $L$ is continuous in $u$ on $\mathcal U$ and has a unique
maximiser in the interior of $\mathcal U$. Uniqueness follows from
the strictly proper scoring rule property of the splice density:
the population maximiser is unique when the residual density is
absolutely continuous and the splice family is identifiable at
$u$.
\item $\sup_{u\in\mathcal U} |L_n(u) - L(u)|
\xrightarrow{P} 0$ (uniform convergence). The integrands are
continuous in $u$ and bounded on compact $u$-sets because the GPD
density is continuous in $(\xi, \sigma)$ at the PWM estimates,
which are continuous in $u$ for $\xi$ bounded away from $1/2$
(PWM identification breaks down at $\xi = 1/2$; in practice we
clip).
\end{enumerate}
Then $\hat u^\star \xrightarrow{P} u^\star_0$ by the standard
argmax-consistency theorem.
\end{proof}

\subsection{Consistency of \tw}
\label{sec:appendix:tw-cons}

\begin{proof}[Sketch of method-invariance (\Cref{prop:method-inv})]
PDHTE consumes only $(Y, T)$, the raw data. The kernel-weighted
top-$k$ identification (\Cref{def:kw-topk}), the Hill plateau
detector (\Cref{def:hill-plateau}), and the DEdH moment estimator
(\Cref{def:dedh-moment}) are all functions of these raw quantities.
Two consistent core methods produce different $\hat\theta(t)$ but
neither affects $(Y, T)$. Therefore both produce identical inputs
to PDHTE and identical outputs.
\end{proof}

\subsection{Sketch: Welsch local-linear DML consistency}
\label{sec:appendix:welsch-cons}

\begin{proof}[Sketch]
\tw\ uses the Welsch weight $w(r)=\exp(-\gamma(r/\sigma)^2)$ at a
\emph{constant} $\gamma$ --- there is no data-dependent tuning in
the second-stage loss. The consistency argument is the standard
two-step DML construction.

\emph{Step 1: nuisance step.} Cross-fit DML conditions
\citep{chernozhukov2018dml} give $\hat g$ converging in
$L^2(P)$ at rate $o_P(n^{-1/4})$. Then
$r^Y_i = Y_i - \hat g(X_i)$ is a $L^2$-consistent estimate of the
true residual $r^{Y,\star}_i = Y_i - g_0(X_i)$.

\emph{Step 2: local-linear step.} Because $\gamma$ is fixed, the
Welsch population score $\psi(r) = r\,e^{-\gamma(r/\sigma)^2}$ is
a fixed bounded, Lipschitz, odd function; its derivative at the
truth is mean-zero in the first-stage nuisance, so the
Neyman-orthogonality condition of \citet[Thm.~3.1]{chernozhukov2018dml}
holds without modification. Standard local-linear M-estimator
theory \citep{fan1996local} then gives, at interior $t_0$ under
twice-differentiability of $\theta_W$ in $t$ and a positive
bounded second derivative of the population score,
$\hat\theta_W(t_0) - \theta_W(t_0) = O_P(n^{-2/5})$.

\emph{Remark (data-dependent $\gamma$).} An earlier variant tied
$\gamma$ to a data-driven $\hat\xi$. There $\gamma$ is a tuning
parameter of the \emph{target functional}, not a first-stage
nuisance: $\hat\xi$ enters the loss, and provided
$\hat\xi \xrightarrow{P}\xi_0$ the estimand stabilises to
$\theta_W(t;\gamma(\xi_0))$, with orthogonality w.r.t.\ the
first-stage nuisances $g,m$ unaffected (the Welsch score's
nuisance-derivative is mean-zero for every fixed $\gamma$).
We nonetheless report the constant-$\gamma$ estimator, which has
no such subtlety and is empirically no worse
(\Cref{sec:experiments}).
\end{proof}

\subsection{Sign-Conditional Recovery of $\E[Y(t)]$}
\label{sec:appendix:ey-derive}

Let $F_{r\mid T=t}$ denote the conditional density of the
residual $r$ given treatment $t$, with support on $\R$. Suppose
the right tail $F_{r\mid T=t}\big(\cdot \mid r > u^\star\big)$ is
GPD with shape $\xi_+$, scale $\sigma_+$ and left tail
$F_{-r\mid T=t}\big(\cdot \mid r < -u^\star\big)$ is GPD with
shape $\xi_-$, scale $\sigma_-$. Let $p_\pm(t) = \Pr(r \gtrless
\pm u^\star \mid T=t)$ and $p_b(t) = 1 - p_+(t) - p_-(t)$.

Decomposing $\E[Y(t)] = \theta_W(t) + \E[r \mid T=t]$:
\begin{align}
\E[r \mid T=t] &= p_b\,\E[r \mid |r| \le u^\star, T=t]
+ p_+\,\E[r \mid r > u^\star, T=t]
+ p_-\,\E[r \mid r < -u^\star, T=t]. \nonumber
\end{align}
The first term is approximately zero under the Welsch fit (the
M-functional centres the bulk). The second term, by the GPD
mean-excess identity (valid for $\xi_+ < 1$), is
$u^\star + \sigma_+ / (1-\xi_+)$. The third term is its mirror:
$-(u^\star + \sigma_-/(1-\xi_-))$. Substituting gives
\Cref{eq:ey-recover}.

For symmetric tails ($p_+ = p_-$, $\xi_+ = \xi_-$,
$\sigma_+ = \sigma_-$), the second and third terms cancel exactly
and $\E[r \mid T=t] = 0$, so $\E[Y(t)] = \theta_W(t)$. For
one-sided positive contamination ($p_- \approx 0$), only the
second term contributes and the recovered $\E[Y(t)]$ shifts above
$\theta_W(t)$ by the GPD positive-tail mean times $p_+$.

\section{Implementation Details}
\label{sec:appendix:implementation}

\Cref{tab:appB-hyperparams} consolidates every hyperparameter
used in the paper; \Cref{tab:appB-nuisance} compares the
nuisance-loss options.

\begin{table}[H]
\centering
\small
\begin{tabular}{llll}
\toprule
\textbf{Component} & \textbf{Parameter} & \textbf{Default} & \textbf{Tested range} \\
\midrule
\tw\ Welsch loss & clipping $\gamma$ & $0.10$ (constant) & $\gamma \in [0.05, 0.20]$ \\
DML & $n_{\rm folds}$ & $3$ & $\{3, 5\}$ \\
DML kernel & bandwidth $h$ & Silverman & --- \\
Threshold rule & grid size $J$ & $40$ & $\{30, 40, 60\}$ \\
Threshold rule & $n_{\min,\rm exc}$ & $30$ & $\{30, 100\}$ \\
Threshold rule & $p_{\rm KS}^{\min}$ & $0.0$ (off) & $\{0, 0.05\}$ \\
PDHTE+JK & $\kappa$-grid & $\{0.04\!:\!0.20\}$ & --- \\
PDHTE+JK & plateau CV threshold & $0.25$ & $\{0.20, 0.25, 0.40\}$ \\
PDHTE+JK & $n_{\rm jk}$ half-samples & $4$ & $\{2, 4, 8\}$ \\
Bulk law (splice) & Laplace scale $b$ & $\sigma_{\rm core}/\sqrt 2$ & --- \\
\bottomrule
\end{tabular}
\caption{Consolidated hyperparameters. ``Tested range'' is the
sensitivity sweep done in development; behaviour was qualitatively
similar across the tested values.}
\label{tab:appB-hyperparams}
\end{table}

\begin{table}[H]
\centering
\small
\begin{tabular}{llll}
\toprule
\textbf{Nuisance loss} & \textbf{Property} & \textbf{Wall-time} & \textbf{Empirical} \\
\midrule
L2 (default) & convex, no robustness & $1\times$ & baseline kurtosis \\
L1 absolute  & bounded influence per split & $\sim 1\times$ & moderate reduction \\
Quantile (median) & max-breakdown M-est. & $\sim 1\times$ & moderate reduction \\
$\gamma$-divergence (Welsch) & redescending & $\sim 30\times$ & best reduction \\
\bottomrule
\end{tabular}
\caption{Phase-1 nuisance loss options. The L2 default is used
for the QR and DML baselines for fairness; the
$\gamma$-divergence GBDT is the default for \tw\ headline
results. The qualitative ranking of \tw\ vs.\ baselines is
unchanged across these choices.}
\label{tab:appB-nuisance}
\end{table}

\subsection{Robust GBDT Nuisance for Phase 1}
\label{sec:robust-nuisance}

The DML nuisance models $\hat g(X) = \hat\E[Y \mid X]$ and
$\hat m(X) = \hat\E[T \mid X]$ are by default
\texttt{HistGradientBoostingRegressor} with squared-error loss.
\tw\ optionally uses one of three robust GBDT losses:

\begin{itemize}[itemsep=2pt,topsep=0pt,leftmargin=1.5em]
\item \textbf{L1 (absolute error)}:
\texttt{HistGradientBoostingRegressor(loss="absolute\_error")}.
Each split minimises $L^1$ loss; bounded influence per split. Fast
(negligible overhead vs.\ default).
\item \textbf{Median (quantile $0.5$)}:
\texttt{HistGradientBoostingRegressor(loss="quantile",
quantile=0.5)}. The maximum-breakdown M-estimator for location.
\item \textbf{$\gamma$-divergence}: custom XGBoost objective with
gradient $g_i = (\hat y_i - y_i) w_i$ and Hessian $h_i = w_i$
where $w_i = \exp(-\gamma\,r_i^2/(2\sigma^2))$ is the Welsch
redescending weight, $\sigma$ MAD-anchored. This is the only
nuisance loss that is redescending rather than merely
bounded-influence. Empirical kurtosis-reduction is best of the
three, at $\sim 30\times$ wall-time cost vs.\ HistGBM (custom
objective in Python rather than the HistGBM C++ kernel).
\end{itemize}

The choice of nuisance loss is orthogonal to the \tw\
contribution. We use the $\gamma$-divergence variant in headline
results; the framework's qualitative behaviour is similar with
the L2 default.

\subsection{Composite-Likelihood Threshold-Rule Defaults}
\label{sec:appendix:lik-thresh-defaults}

Defaults used throughout:
\begin{itemize}[itemsep=2pt,topsep=0pt,leftmargin=1.5em]
\item Grid: $J = 40$ candidate quantiles uniformly spaced in
$[q_{0.5}, q_{1 - n_{\min}/n}]$.
\item Minimum exceedance budget: $n_{\min, \mathrm{exc}} = 30$
(default); $= 100$ when stability is preferred over coverage.
\item KS validity gate: $p_{\mathrm{KS}}^{\min} = 0$ by default
(disabled); set to $0.05$ for stricter refusal behaviour.
\item Bulk law: Laplace, with scale $b = \sigma_{\rm core}/\sqrt 2$
and $\sigma_{\rm core} = 1.4826 \cdot \mathrm{MAD}_{|r| \le u}$.
\item PWM on exceedances $\{e_j\}$: $\hat\xi^{\rm PWM}$,
$\hat\sigma^{\rm PWM}$ from \Cref{eq:pwm}.
\end{itemize}

\subsection{Per-$T$ Tail Shape Bandwidth}
\label{sec:appendix:per-t-bandwidth}

The kernel bandwidth $h_T$ in \Cref{def:per-t-xi} uses Silverman's
rule applied to the observed $T$-distribution:
$h_T = 1.06\,\sigma_T\,n^{-1/5}$. The effective per-$t$
exceedance count is
$n_{\rm exc}^{\rm eff}(t) = \sum_{i: |r_i| > u^\star}
\exp(-((T_i - t)/h_T)^2/2)$. When $n_{\rm exc}^{\rm eff}(t) < 15$,
$\hat\xi(t)$ is set to NaN and the global $\hat\xi$ is used in its
place for the recovery formula.

\subsection{Hybrid Empirical-Plus-GPD $\hat Q_\alpha(t)$ Switch}
\label{sec:appendix:Q-hybrid}

The crossover $\alpha\,n_{\rm eff}(t) \ge 1$ determines whether
the empirical branch or GPD branch is used. The empirical
$(1-\alpha)$-quantile from the kernel-weighted residual sample is
$r^{(j)}$ where $j = \min\{i : \hat F_w(r_i) \ge 1-\alpha\}$ and
$\hat F_w$ is the weighted empirical CDF. The GPD branch uses
\Cref{eq:return-level} substituting $\hat\xi(t), \hat\sigma(t)$
for the global $\hat\xi, \hat\sigma$.

\subsection{Hyperparameters in \tw}
\label{sec:appendix:tw-params}

\tw\ uses a constant Welsch clipping $\gamma = 0.10$ --- a mild
redescending downweighting. The value was chosen by smoke-tests
on a handful of DGPs; core-ADRF MAE is qualitatively unchanged
for $\gamma \in [0.05, 0.20]$. Tying $\gamma$ to a data-driven
$\hat\xi$ was tested and gave no improvement at the $320$-cell
scale, so the constant value is the default (this also keeps the
second-stage loss free of data-dependent tuning; see
\Cref{sec:appendix:welsch-cons}).

\section{Supplementary Experiments and Figures}
\label{sec:appendix:supplementary}

This appendix provides per-DGP breakdowns, additional comparisons,
diagrams, and result figures referenced from the main body.

\subsection{Per-$T$ Tail-Shape Recovery}
\label{sec:appendix:per-t-xi}

\begin{figure}[H]
\centering
\includegraphics[width=0.66\linewidth]{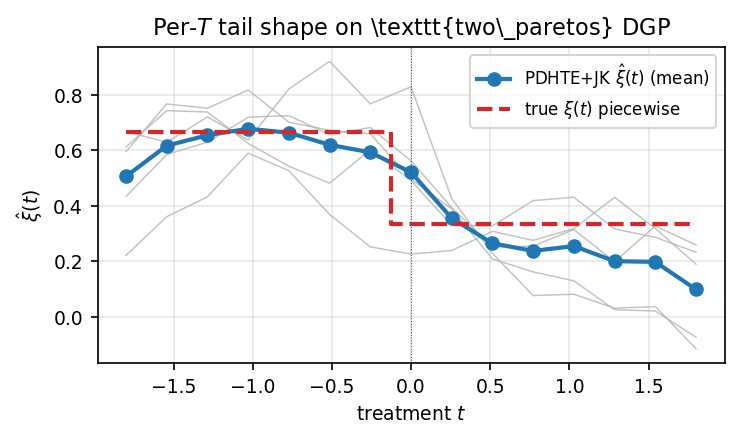}
\caption{PDHTE+JK $\hat\xi(t)$ on \texttt{sinusoidal\_two\_paretos}
($n=3000$, $p=0.10$): $5$ seeds (grey), mean (blue),
true piecewise $\xi(t)$ (dashed red). The jump at $t = 0$ is
recovered with $\sim 0.05$ residual bias; ordering across $t$ is
reliable in every seed. Quantitative summary in
\Cref{tab:per-t-xi}.}
\label{fig:per-t-xi}
\end{figure}

\subsection{freMTPL2 Downstream Metrics}
\label{sec:appendix:fremtpl2}

\begin{table}[H]
\centering
\small
\begin{tabular}{lrr}
\toprule
\textbf{Method} & \textbf{MAE @ $\alpha\!=\!0.01$}
                & \textbf{MAE @ $\alpha\!=\!0.001$} \\
\midrule
QR (Koenker)                  & $\mathbf{6.87}$ & $7.17$ \\
Residual-PWM POT              & $7.10$ & $\mathbf{6.14}$ \\
Proposed $\hat Q_\alpha(t)$    & refuse & refuse \\
\bottomrule
\end{tabular}
\caption{freMTPL2 return-level MAE ($n_{\rm train}=5000$, $5$
resample seeds, $\log(Y+1)$ scale). At the deep-tail
$\alpha = 0.001$ the residual-PWM POT beats QR by $14\%$;
PDHTE+JK refuses at every grid point, correctly flagging that
log-claim severity fails the regular-varying plateau test.}
\label{tab:downstream-fremtpl2}
\end{table}

\subsection{Return Level under $T$-$X$ Confounding}
\label{sec:appendix:confounded}

\begin{figure}[H]
\centering
\includegraphics[width=0.7\linewidth]{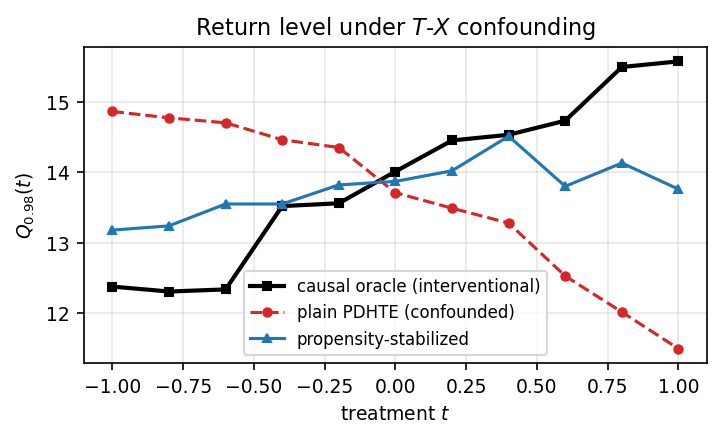}
\caption{Return-level curve $\hat Q_{0.98}(t)$ on the confounded
DGP (\Cref{sec:exp:confounded}). Plain PDHTE (red) tracks the
\emph{observational} $Y\mid T{=}t$ tail and gets the slope
backwards --- decreasing where the causal curve increases ---
because high $T$ selects the contamination-free $X_0>0$ stratum.
Propensity-stabilized weighting (blue) restores the marginal
$X$-distribution in each $T$-band and recovers the interventional
oracle (black), cutting return-level bias by $63\%$. The tail
\emph{index} (\Cref{prop:index-robust}) needs no such correction.}
\label{fig:confounded-q}
\end{figure}

\begin{table}[H]
\centering
\small
\begin{tabular}{lrr}
\toprule
\textbf{Target} & \textbf{Plain MAE} & \textbf{Propensity-stabilized MAE} \\
\midrule
Tail index $\hat\xi(t)$           & $0.028$ & $0.028$ \\
Return level $\hat Q_{0.98}(t)$   & $1.93$  & $\mathbf{0.72}$ \\
\bottomrule
\end{tabular}
\caption{Confounded DGP ($n = 6000$, $8$ seeds), MAE against the
interventional (causal) oracle. The tail index is unaffected by
confounding (\Cref{prop:index-robust}); the return level is
confounded, and propensity-stabilized weighting cuts its bias by
$63\%$.}
\label{tab:confounded}
\end{table}

\subsection{Sample-Size Scaling (Per-DGP)}
\label{sec:appendix:sample-size-detail}

\begin{figure}[H]
\centering
\includegraphics[width=0.95\linewidth]{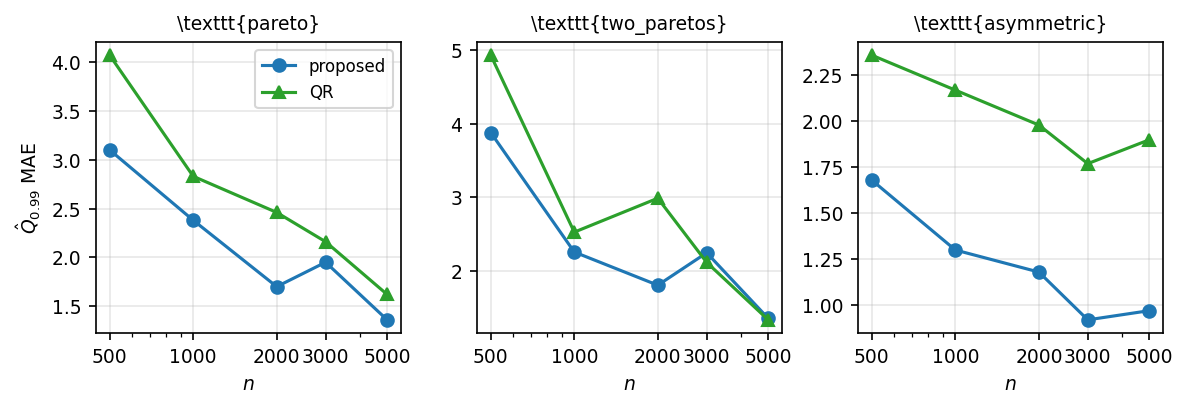}
\caption{Per-DGP MAE of $\hat Q_{0.99}(t)$ vs.\ $n$
($5$ seeds each). The proposed estimator beats QR at every $n$
on \texttt{asymmetric}; on the homogeneous-tail
\texttt{pareto} and \texttt{two\_paretos}, the QR curves cross
the proposed estimator around $n \approx 2000$--$3000$. The
per-DGP table below adds the residual-PWM POT baseline.}
\label{fig:sample-size}
\end{figure}

\begin{table}[H]
\centering
\small
\begin{tabular}{rrrr|rrr|rrr}
\toprule
& \multicolumn{3}{c|}{\texttt{pareto}}
& \multicolumn{3}{c|}{\texttt{two\_paretos}}
& \multicolumn{3}{c}{\texttt{asymmetric}} \\
\textbf{$n$}
 & Prop. & RPWM & QR
 & Prop. & RPWM & QR
 & Prop. & RPWM & QR \\
\midrule
$500$  & $5.27$ & $\mathbf{3.73}$ & $4.94$
        & $4.75$ & $\mathbf{3.88}$ & $4.93$
        & $6.40$ & $\mathbf{1.68}$ & $2.36$ \\
$1000$ & $5.00$ & $\mathbf{3.57}$ & $3.81$
        & $3.89$ & $\mathbf{2.26}$ & $2.53$
        & $\mathbf{1.24}$ & $1.30$ & $2.17$ \\
$2000$ & $3.68$ & $\mathbf{2.11}$ & $2.40$
        & $3.10$ & $\mathbf{1.81}$ & $2.99$
        & $\mathbf{0.91}$ & $1.18$ & $1.98$ \\
$3000$ & $3.76$ & $2.66$ & $\mathbf{2.60}$
        & $3.13$ & $2.25$ & $\mathbf{2.12}$
        & $\mathbf{0.90}$ & $0.92$ & $1.77$ \\
$5000$ & $3.82$ & $1.74$ & $\mathbf{1.63}$
        & $2.78$ & $1.36$ & $\mathbf{1.34}$
        & $\mathbf{1.01}$ & $0.97$ & $1.90$ \\
\bottomrule
\end{tabular}
\caption{Full per-DGP sample-size scaling: proposed
(PDHTE+JK-driven), residual-PWM POT (RPWM), and QR. Below
$n=2000$ the parametric tail prior in either RPWM or the
proposed estimator dominates QR. The proposed estimator's
PDHTE+JK refusal mechanism is conservative at very small $n$;
RPWM is the practical small-$n$ default on regular-varying
tails because PDHTE+JK over-refuses when plateau noise dominates.}
\label{tab:sample-size-scaling-full}
\end{table}

\subsection{Conditional Shortfall (Per-DGP)}
\label{sec:appendix:shortfall-detail}

\begin{figure}[H]
\centering
\includegraphics[width=0.85\linewidth]{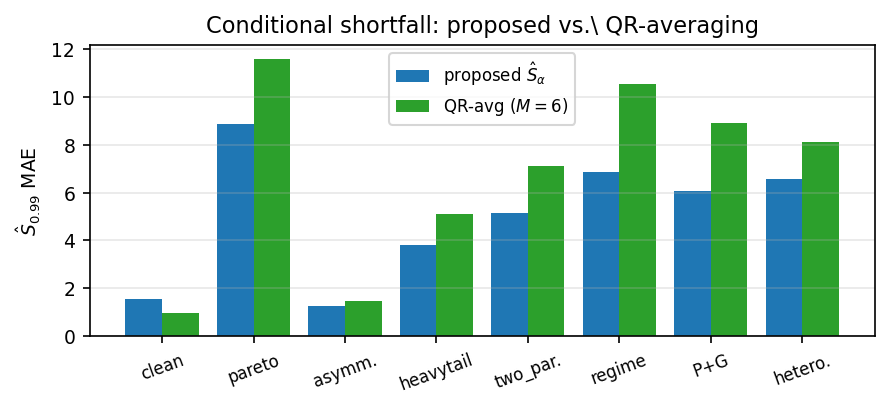}
\caption{Per-DGP $\hat S_{0.99}(t)$ MAE: proposed closed-form
(blue) vs.\ QR-averaging ($M=6$) (green). The proposed
estimator wins on every heavy or contaminated DGP; QR-averaging
wins only on the clean Gaussian baseline.}
\label{fig:shortfall-bars}
\end{figure}

\end{document}